\documentclass{article}

\usepackage{microtype}
\usepackage{graphicx}
\usepackage{subcaption}
\usepackage{booktabs} 
\usepackage{hyperref}
\usepackage{bm}

\usepackage[preprint]{icml2026}

\usepackage{amsmath}
\usepackage{amssymb}
\usepackage{mathtools}
\usepackage{amsthm}
\usepackage{float}

\expandafter\let\csname algorithm*\endcsname\relax
\expandafter\let\csname endalgorithm*\endcsname\relax

\makeatother
\usepackage[ruled, vlined]{algorithm2e}
\usepackage[capitalize,noabbrev]{cleveref}

\theoremstyle{plain}
\newtheorem{theorem}{Theorem}[section]
\newtheorem{proposition}[theorem]{Proposition}
\newtheorem{lemma}[theorem]{Lemma}
\newtheorem{corollary}[theorem]{Corollary}
\theoremstyle{definition}
\newtheorem{definition}[theorem]{Definition}
\newtheorem{assumption}[theorem]{Assumption}
\theoremstyle{remark}
\newtheorem{remark}[theorem]{Remark}

\usepackage[textsize=tiny]{todonotes}

\icmltitlerunning{GMHF}
\usepackage{tikz}
\usetikzlibrary{positioning}
\tikzset{
    datanode/.style={draw, rectangle, fill=white!30, minimum width=4cm, align=center},
    deploynode/.style={draw, rectangle, fill=orange!30, minimum width=4cm, align=center},
    humaninsightnode/.style={draw, ellipse, fill=green!30, minimum width=4cm, align=center},
    cvaenode/.style={draw, rectangle, fill=white!20, minimum width=3cm, align=center},
    agentnode/.style={draw, circle, fill=white!20, minimum size=1.5cm, align=center},
    metanode/.style={draw, rectangle, fill=green!20, minimum width=3cm, align=center},
    feedbacknode/.style={draw, ellipse, fill=orange!20, minimum width=2.5cm, align=center}
}
\begin{document}

\twocolumn[
  \icmltitle{Human-Machine Collaboration on Generative Meta-Learning: Model and Algorithm}




\begin{icmlauthorlist}
    \icmlauthor{Midhun Parakkal Unni}{unishef,cmi,uniman}
    \icmlauthor{Samuel Kaski}{uniman,ellis,aalto}
\end{icmlauthorlist}
\icmlaffiliation{unishef}{School of Computer Science, University of Sheffield, Sheffield, UK}
\icmlaffiliation{cmi}{Centre for Machine intelligence, University of Sheffield, UK}
\icmlaffiliation{uniman}{Department of Computer Science, University of Manchester, Manchester, UK}
\icmlaffiliation{ellis}{ELLIS Institute Finland, Helsinki, Finland}
\icmlaffiliation{aalto}{Department of Computer Science, Aalto University, Espoo, Finland}
\icmlcorrespondingauthor{Midhun Parakkal Unni}{m.parakkalunni@sheffield.ac.uk}

  \icmlkeywords{Human in the loop Machine learning, Probabilistic modelling}

  \vskip 0.3in
]



\printAffiliationsAndNotice{}  

\begin{abstract}
Generalizing machine learning models to environments that differ from their training distribution remains a critical hurdle, particularly when data from the target domain is entirely or partially unavailable. We propose Generative Meta-Learning with Human Feedback (GMHF), a novel framework that bridges this domain gap by leveraging expert intuition to guide data synthesis. Grounded in a theoretical analysis of generalization error, we derive bounds demonstrating that aligning the distribution of generated data with human beliefs regarding the target physics significantly mitigates risk. GMHF operationalizes this insight by employing a Conditional Neural ODE (cNODE) as a generative digital twin, coupled with a Reinforcement Learning (RL) agent. The agent iteratively refines the latent physical parameters of the generated trajectories based on feedback, effectively steering the meta-learner toward the unobserved target distribution. Empirical validation on a nonlinear Duffing oscillator shows that GMHF substantially reduces deployment loss as expert reliability increases, and that the divergence between generated and target data falls under reliable feedback, directly corroborating the divergence-minimisation mechanism predicted by our theory. Further experiments on a non-dynamical probabilistic model confirm that the framework extends beyond ODE-governed systems, establishing human-AI collaboration as a rigorous catalyst for robust generalisation under distribution shift.
\end{abstract}
\section{Introduction}
Artificial intelligence and humans increasingly collaborate to solve real-world problems, either explicitly or implicitly \cite{reverberi2022experimental, mosqueira2023human}. Making human-AI collaboration explicit enables scientific predictions about the behaviour of these collaborative systems and the possibility of systematically controlling them. In this work, we address a central challenge in machine learning—enabling trained models to generalise to unseen data distributions at deployment—by leveraging human-AI collaboration. Empirical surveys show that incorporating human expertise into active learning loops enhances model performance, particularly by mitigating data scarcity through the strategic selection of informative samples \cite{budd2021survey}. Given recent advances in generative models \cite{lin2024diffusion}, `leveraging interactions with generative models to probe expert knowledge' emerges as a promising technique \cite{christiano2017deep}. This highlights the need for a principled investigation of such systems and the development of algorithms that integrate generative models with human-AI collaboration.

The standard independent and identically distributed data (i.i.d.) assumption in machine learning (ML) \cite{loh2017lower} has been increasingly relaxed to facilitate the study of distributional shifts \cite{ben2010theory, liu2021towards}. Recent research approaches the distribution shift from multiple perspectives, such as utilising causal structures, identifying invariant features across domains \cite{arjovsky2019invariant}, and addressing selection bias \cite{liu2021towards}. However, the explicit role of human collaboration remains largely unaddressed in most existing methods.

In particular, distribution shifts in time series data deserve special attention \cite{wu2025out} because (a) they occur across diverse domains, including healthcare and financial analysis, and (b) the temporal structure allows for modelling them via dynamical systems and related causal structures \cite{chen2018neural, rubanova2019latent, parakkal2023modelling, capinski2012black}(e.g., kinematics of human limbs to Black–Scholes models). Furthermore, substantial human expertise is available for a subset of time-series tasks (e.g., medical doctors interpreting ECG data), which can be leveraged to enhance model performance \cite{tonekaboni2019clinicians}.

Our work introduces a mathematical model of human-generative AI collaborative systems and derives theoretical bounds that characterise their behaviour, subject to assumptions. We then analyse a Gaussian approximation of the model to obtain closed-form solutions. Building on these theoretical insights, we propose a practical algorithm and evaluate its performance in simulated scenarios. While our primary evaluation is on dynamical time series, we additionally verify that the framework applies without modification to a non-dynamical probabilistic model, indicating that the approach is not specific to ODE-governed systems.

We introduce a novel approach to out-of-distribution generalisation that leverages human feedback during dataset generation, guided by a dynamic system. Our method, Generative Meta-Learner with Human Feedback (GMHF), integrates generative modelling, reinforcement learning (RL), and human feedback to optimise the quality of generated datasets. By incorporating human evaluations into dataset refinement, we aim to improve the generalisation performance of models in partially observed deployment scenarios.

\paragraph{GMHF framework in detail}
The GMHF framework comprises three key components: a generative model, an AI-agent, and a meta-learner. We use a conditional neural ODE (cNODE) architecture that learns a dynamical system from simulated time-series data (e.g., a Duffing oscillator), thereby serving as a data generator. While there are several approaches to conditioning a neural ODE \cite{dang2023conditional, wang2025conditional}, we take an approach that creates a disentangled action space for the agent (a reinforcement learning (RL) agent).
The generative model produces candidate training data. The AI-agent refines these datasets by transforming the latent variables, guided by human feedback. This feedback is incorporated into the generation loop, enabling the AI-agent to iteratively improve dataset quality, which the meta-learner then uses to learn predictive models.

The GMHF approach is particularly relevant in time series analysis, where the underlying dynamics may change over time or across environments. By leveraging human feedback, we aim to enhance the robustness of machine learning models to distribution shifts, ultimately improving their performance in real-world applications. The main contributions of this paper are as follows:
\begin{itemize}
    \item We propose a mathematical model of the human feedback-based data selection process and derive theoretical bounds that highlight its potential to improve generalisation in deployment scenarios.
    \item We introduce the Generative Meta-Learner with Human Feedback (GMHF) framework, which integrates generative modelling, reinforcement learning, and human feedback to refine datasets.
    \item We empirically demonstrate the effectiveness of GMHF on a nonlinear Duffing oscillator, and show that the divergence between generated and target distributions decreases under reliable feedback, empirically corroborating the divergence-minimisation mechanism identified by our theory.
    
    \item We further show that GMHF generalises beyond ODE-governed time series by applying it, without architectural modification, to a non-dynamical probabilistic model with an explicit source-to-target extrapolation gap (Appendix \ref{sec:pgm_benchmark}).
\end{itemize}

\section{Mathematical Modelling}
\label{sec:math_modelling}
\subsection{Problem Setting}
\label{sec:setting}

We consider a deployment scenario where a generative model is used to synthesise training data for an unobserved target environment. The process is governed by a human-in-the-loop protocol: the generator proposes a candidate trajectory (e.g., a time-series simulation), and a human expert evaluates its plausibility against their mental model of the target physics. The expert acts as a filter, either accepting the sample if it aligns with the target dynamics or rejecting it if it diverges from them. This selective acceptance shapes the final distribution of training data and, therefore, the deployment accuracy, effectively aligning the generative process with the latent ground truth through iterative feedback.

\subsection{Modelling}
In this section, we develop a theoretical framework to characterise the interaction between the generative model and the human agent. Specifically, we propose a set of assumptions formalising the generative distribution, the human-based likelihood function, and the associated priors.

We begin by defining the relationship between the generative model and the target deployment distribution.
\begin{assumption}[Common Support]
Let $D_g$ be the generative distribution and $D_t$ be the target distribution with densities $p_g(\cdot)$ and $p_t(\cdot)$ respectively. We assume that the generative model covers the support of the target distribution, such that $p_t(x)>0$ implies $p_g(x)>0$ for all $x \in \mathcal{X}$.
\end{assumption}

The generative model is conditional and parametrised by $\theta$. We model the human agent as a filter that accepts a subset of the generated data based on a utility function $p_{u_{\tilde{\theta}}}(\cdot)$, which represents their belief about the latent structure of the deployment scenario. This yields the following distribution for the accepted samples:
\begin{assumption}[Human Acceptance]
Given generative density $p_g(x \mid \theta)$, the human agent accepts samples proportional to a utility function $p_{u_{\tilde{\theta}}}(x)$. This utility is parametrised by $\tilde{\theta}$, representing the human's belief about the true latent dynamics, yielding the accepted distribution:
\begin{align}
p_a(x \lvert \theta) := \frac{p_{u_{\tilde{\theta}}}(x) \, p_g(x \lvert \theta)}{z_a}
\label{eq:acceptance distribution}
\end{align}
\end{assumption}
We assume that human utility is grounded in their understanding of the domain's latent structure.
\begin{assumption}[Human Mental Model]
The human utility $p_{u_{\tilde{\theta}}}(x)$ decomposes into a prior belief $\tilde{p}_t(x)$ over the target distribution and a likelihood $p_h(\tilde{\theta} \mid x)$, representing the human's inference of the latent parameters given the data:
\begin{align}
    p_{u_{\tilde{\theta}}}(x) := \frac{p_h(\tilde{\theta} \mid x) \tilde{p}_t(x)}{z_h}
    \label{eq:human user model}
\end{align}  
where $z_h$ is a normalisation constant.
\end{assumption}
\begin{definition}
A hypothesis is a function defined from the set $\mathcal{X}$ to $\mathcal{Y}$, $h: \mathcal{X} \to \mathcal{Y}$, where $\mathcal{X}$, $\mathcal{Y}$ are sets defining the inputs and labels, respectively.
\end{definition}

\begin{definition}
The error of a hypothesis $h$ according to the distribution $D_i$ and labelling function $f$ is given by:
\begin{align}
\epsilon_{x}(h, f) := \mathbb{E}_{x \sim D_i}[\lvert h(x) - f(x) \rvert]
\end{align}
\end{definition}

This leads to the following proposition for the density of the accepted samples.
\begin{proposition}
The density of the accepted samples is given by,
\begin{align}
p_a(x):=p_a(x \lvert \theta) := \frac{p_h(\tilde{\theta} \lvert x) \tilde{p}_t(x)p_g(x \lvert \theta)}{z} 
\end{align}
\label{prop:acceptance dist}
\end{proposition}

\begin{proof}
Obtained by substituting (\ref{eq:human user model}) in (\ref{eq:acceptance distribution}) and choosing an appropriate normalising constant.
\end{proof}
The divergence between the probability distribution of the accepted samples and the target probability distribution is provided in the following lemma. Proofs of the theorems are given in the Appendix \ref{Appendix:Proofs}.

\begin{lemma}
\begin{align}
    D_{KL}(p_t \| p_a) &= 
\underbrace{\delta_h}_{\text{Human belief mismatch}} \nonumber \\ 
&-\underbrace{L_h}_{\text{Human latent inference}} \nonumber \\
&-\underbrace{L_g}_{\text{Generator data fit}} \nonumber \\ 
&+\underbrace{\log z}_{\text{Normalization}}
\end{align}

where,
\begin{align}
   L_h &:=\mathbb{E}_{p_t}[\log p_h(\tilde{\theta}|x)] \\
   L_g &:=\mathbb{E}_{p_t}[\log p_g(x|\theta)]\\
   \delta_h &:=D_{KL}(p_t \| \tilde{p}_t)
\end{align}
\label{eq:KL expression}
\end{lemma}

The estimated divergence between $p_t$ and $p_a$ provides a useful means to bound the target error for a hypothesis $h$, denoted as $\epsilon_t(h)$. According to the decomposition in Lemma \ref{eq:KL expression}, this divergence is minimized under the following conditions: (1) human belief closely matches the true target distribution (minimising $\delta_h$); (2) the human's internal model is well aligned with the true target observations (maximising $L_h$); and (3) the generative model accurately fits the target data (maximising $L_g$).
\begin{lemma}
For a hypothesis $h$, the target error $\epsilon_t(h)$ and the source error $\epsilon_s(h)$, one could bound the target error as follows. This lemma is based on \cite{ben2010theory}.
    \begin{align}
        \epsilon_t(h) &\leq \epsilon_s(h) + \min \{ \mathbb{E}_{p_a}[\mid f_s(x) - f_t(x)\mid ],\nonumber \\
        &\quad \mathbb{E}_{p_t}[\mid f_s(x) - f_t(x)\mid ]\} + \mid \mathbb{E}_{p_t}(h(x)-f_t(x)) \nonumber \\ 
        &- \mathbb{E}_{p_a}(h(x)-f_t(x)) \mid
    \end{align}
    \label{eq: shai dist shift}
\end{lemma}
\begin{lemma}[Bounding the expectation term]
Let $p_t(x)$ be the target distribution and $p_a(x)$ the acceptance distribution as defined in the previous assumptions. Let $h: \mathcal{X} \to \mathbb{R}$ be a bounded function such that $|h(x) - f_t(x)| \leq M$ for all $x \in \mathcal{X}$. Then, the expectation mismatch term
\begin{align}
    \Delta := \mathbb{E}_{p_t}[(h(x) - f_t(x))] - \mathbb{E}_{p_a}[(h(x) - f_t(x))]
\end{align}
can be bounded as follows:
\begin{align}
    |\Delta| \leq M \sqrt{2 D_{KL}(p_t \| p_a)}.
\end{align}
\label{eq: bound expectation difference}
\end{lemma}

\begin{theorem}
For a hypothesis $h$, the target error $\epsilon_t(h)$ and the source error $\epsilon_s(h)$, one can modify the bound for the target error as follows, where, $h: \mathcal{X} \to \mathbb{R}$ be a bounded function such that $|h(x) - f_t(x)| \leq M$
    \begin{align}
        \epsilon_t(h) &\leq  \epsilon_s(h) +\min \{ \mathbb{E}_{p_a}[\mid f_s(x) - f_t(x)\mid ], \nonumber \\
        &\quad \mathbb{E}_{p_t}[\mid f_s(x) - f_t(x)\mid ]\} \nonumber \\
        &\quad+M\sqrt{2 (\delta_h - L_h - L_g + \log z)}
    \end{align}
\label{thm:final bound}
\end{theorem}
\paragraph{Takeaways from the general model:}
One could minimise the deployment error by doing one or a combination of the following,  
\begin{itemize}
    \item Improving human understanding of the deployment data, that is, by reducing the human belief mismatch.
    \item Accurate estimation of the latent variable that controls the deployment data. The human's acceptance of a data point depends on the human's belief about the correct distribution of the data, and also on how the latent variable controls the generation of the trainable datasets.
    \item Generator accurately fitting the data.
    \item One should also note that terms in the bound can cancel out and compensate each other; therefore, learning the user model or personalising the algorithms to a user can be helpful.
\end{itemize}

\subsection{Gaussian special case}
\label{sec:gaussian_model}

We first consider the case where the user model, the generative model, and the target distribution are all multivariate Gaussian. This permits closed-form analysis of the accepted distribution and associated divergence measures.

\begin{proposition}[Accepted Distribution under Gaussian Models]
Let the user model be \( p_{u_{\tilde{\theta}}}(x) = \mathcal{N}(x; \mu_h, \Sigma_h) \), the generative model be \( p_g(x) = \mathcal{N}(x; \mu_g, \Sigma_g) \), and the target distribution be \( p_t(x) = \mathcal{N}(x; \mu_t, \Sigma_t) \). Then the distribution of accepted samples is also Gaussian:
\[
p_a(x) = \mathcal{N}(x; \mu_a, \Sigma_a),
\]
where
\begin{align}
\Sigma_a &= \left( \Sigma_h^{-1} + \Sigma_g^{-1} \right)^{-1}, \\
\mu_a &= \Sigma_a \left( \Sigma_h^{-1} \mu_h + \Sigma_g^{-1} \mu_g \right).
\end{align}
\end{proposition}

\begin{proof}
Follows from the identity for the product of two Gaussian densities. The resulting unnormalised product is proportional to a Gaussian with the given mean and covariance.
\end{proof}

\begin{proposition}[Optimal Generative Parameters to Match Target]
Suppose we require the accepted distribution \( p_a \) to exactly match the target distribution \( p_t = \mathcal{N}(\mu_t, \Sigma_t) \). Then, for a given user model \( (\mu_h, \Sigma_h) \), the generative model parameters \( (\mu_g, \Sigma_g) \) that achieve this are:
\begin{align}
\Sigma_g &= \left( \Sigma_t^{-1} - \Sigma_h^{-1} \right)^{-1}, \\
\mu_g &= \Sigma_g \left( \Sigma_t^{-1} \mu_t - \Sigma_h^{-1} \mu_h \right).
\end{align}
\label{eq:correct_gen_dist_gaus}
\end{proposition}

\begin{proof}
Substitute \( \Sigma_a = \Sigma_t \) and \( \mu_a = \mu_t \) into the formulas from the previous proposition and solve for \( \Sigma_g \) and \( \mu_g \). The inverse \( (\Sigma_t^{-1} - \Sigma_h^{-1}) \) must exist and be positive definite.
\end{proof}

\begin{proposition}[KL Divergence from Target to Accepted Distribution]
Let \( p_t = \mathcal{N}(\mu_t, \Sigma_t) \) and \( p_a = \mathcal{N}(\mu_a, \Sigma_a) \). The KL divergence is:
\begin{align}
D_{\text{KL}}(p_t \,\|\, p_a) &= \log \frac{|\Sigma_a|}{|\Sigma_t|} - k 
+ \operatorname{tr}(\Sigma_a^{-1} \Sigma_t) \nonumber \\
&\quad + (\mu_a - \mu_t)^T \Sigma_a^{-1} (\mu_a - \mu_t),
\end{align}
where \( k \) is the dimensionality of \( x \).
\end{proposition}

\begin{proof}
This is the standard formula for the KL divergence between two multivariate Gaussian distributions.
\end{proof}

\begin{corollary}[KL Divergence in the Isotropic Case]
Let all distributions be isotropic:
\[
p_t = \mathcal{N}(\mu_t, \sigma_t^2 I), \quad
p_{u_{\tilde{\theta}}} = \mathcal{N}(\mu_h, \sigma_h^2 I), \quad
p_g = \mathcal{N}(\mu_g, \sigma_g^2 I).
\]
Then the accepted distribution has the following parameters:

\[
\mu_a 
= \sigma_a^2 \left( \frac{\mu_h}{\sigma_h^2} + \frac{\mu_g}{\sigma_g^2} \right),
\quad
\sigma_a^2 = \frac{\sigma_h^2 \sigma_g^2}{\sigma_h^2 + \sigma_g^2}.
\]
and the KL divergence becomes:
\begin{align}
D_{\text{KL}}(p_t \,\|\, p_a) &=
\frac{k}{2} \left[
\log \left( \frac{\sigma_t^2 (\sigma_h^2 + \sigma_g^2)}{\sigma_h^2 \sigma_g^2} \right)
- 1 + \frac{\sigma_h^2 \sigma_g^2}{\sigma_t^2 (\sigma_h^2 + \sigma_g^2)}
\right] \nonumber\\&\quad
+ \frac{1}{2 \sigma_a^2} \| \mu_a - \mu_t \|^2.
\label{eq:kl_gaus}
\end{align}
\end{corollary}

\begin{proof}
Follows from substituting isotropic forms into the general KL divergence formula and simplifying trace and determinant terms.
\end{proof}

\begin{proposition}[Expected Disagreement between Linear Labelling Functions]
Let \( f_s(x) = w_s^\top x + b_s \) and \( f_t(x) = w_t^\top x + b_t \) be scalar-valued linear labelling functions defined on inputs \( x \in \mathbb{R}^d \), and suppose the input distribution is Gaussian: \( x \sim \mathcal{N}(\mu_a, \Sigma_a) \). Define the parameter differences
\[
\Delta_w := w_s - w_t, \quad \delta_b := b_s - b_t
\]
Then the difference in label predictions is:
\[
z := f_s(x) - f_t(x) = \Delta_w^\top x + \delta_b \sim \mathcal{N}(\mu_{\text{tran}_a}, \sigma_{\text{tran}_a}^2)
\]
where:
\[
\mu_{\text{tran}_a} = \Delta_w^\top \mu_a + \delta_b, \quad \sigma_{\text{tran}_a}^2 = \Delta_w^\top \Sigma_a \Delta_w
\]
 
The expected disagreement under the input distribution \( p_a(x) \) is
\begin{align}
\mathbb{E}_{x \sim p_a}[|f_s(x) - f_t(x)|] = \sigma_{\text{tran}_a} \sqrt{\frac{2}{\pi}} \exp\left( -\frac{\mu_{\text{tran}_a}^2}{2 \sigma_{\text{tran}_a}^2} \right)
+ \nonumber\\
\mu_{\text{tran}_a} \left(1 - 2 \Phi\left(-\frac{\mu_{\text{tran}_a}}{\sigma_{\text{tran}_a}} \right) \right)
\label{eq:disagreement_pa}
\end{align}
where \( \Phi \) is the cumulative distribution function (CDF) of the standard normal distribution:
\[
\Phi(z) = \frac{1}{\sqrt{2\pi}} \int_{-\infty}^{z} e^{-t^2/2} \, dt.
\]
\label{prop:labelling disagreement}
\end{proposition}
\vspace{-10pt}
\begin{proof}
Follows from computing the expected value of the absolute value of a univariate Gaussian variable \( z \sim \mathcal{N}(\mu_{\text{tran}_a}, \sigma_{\text{tran}_a}^2) \).
\end{proof}
\begin{remark}
In this approximation, the `concepts' $f_s$ and $f_t$ are assumed to be deterministic. A non-deterministic case could be accommodated with the additional term $\Lambda$ as follows:
\[
\mu_{\text{tran}_a} = \Delta_w^\top \mu_a + \delta_b, \quad \sigma_{\text{tran}_a}^2 = \Delta_w^\top \Sigma_a \Delta_w + \Lambda
\]
where $\Lambda$ is the covariance term.
\end{remark}

\begin{remark}
With a similar argument as given in Prop. \ref{prop:labelling disagreement} one could obtain the expected disagreement over the target distribution as follows:

\begin{align}
\mathbb{E}_{x \sim p_t}[|f_s(x) - f_t(x)|] = \sigma_{\text{tran}_t} \sqrt{\frac{2}{\pi}} \exp\left( -\frac{\mu_{\text{tran}_t}^2}{2 \sigma_{\text{tran}_t}^2} \right)
+ \nonumber\\
\mu_{\text{tran}_t} \left(1 - 2 \Phi\left(-\frac{\mu_{\text{tran}_t}}{\sigma_{\text{tran}_t}} \right) \right)
\label{eq:disagreement_pt}
\end{align}
where,
\[
\mu_{\text{tran}_t} = \Delta_w^\top \mu_t + \delta_b, \quad \sigma_{\text{tran}_t}^2 = \Delta_w^\top \Sigma_t \Delta_w
.\]
Subsequently, by substituting Eq. \ref{eq:kl_gaus}, \ref{eq:disagreement_pa}, \ref{eq:disagreement_pt} in \ref{eq:useful_for_gaus_bound} one obtains the bound for the special Gaussian case. 
\end{remark}
The key takeaway from the isotropic case is that one must adjust the generative model in accordance with the human model to minimise the bound. The Proposition \ref{eq:correct_gen_dist_gaus} shows that $\Sigma_g$ and $\mu_g$ depend on both the target and human models. Therefore, the aim of the generative model is not to match the target distribution exactly, but rather to collaborate with the human to create an accepted distribution that approximates it.
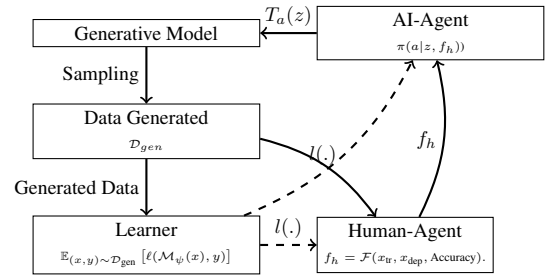
\begin{figure}[b]
\centering
\begin{tikzpicture}[ scale=0.75, transform shape]
\node[datanode] (cvae) {Generative Model}; 
\node[datanode, right=of cvae] (agent) {AI-Agent \\
\tiny \( \pi(a|z, f_h) \))};
\node[datanode, below=of cvae] (data) {Data Generated\\
\tiny $\mathcal{D}_{gen}$};
\node[datanode, below=of data] (meta) {Learner\\ \tiny $\mathbb{E}_{(x, y) \sim \mathcal{D}_{\text{gen}}} \left[\ell(\mathcal{M}_\psi(x), y)\right]$};
\node[cvaenode, right=of meta] (hagent) {Human-Agent \\
\tiny $ f_h = \mathcal{F}(x_{\text{tr}}, x_{\text{dep}}, \text{Accuracy}).$};

\draw[->, thick] (agent) --  node[above]{ $T_a(z)$} (cvae);
\draw[->, thick] (cvae) -- node[left]{Sampling} (data);
\draw[->, thick] (data) -- node[left]{Generated Data} (meta);
\draw[->, thick] (hagent) to [bend right=20] node[left]{$f_h$} (agent);

\draw[->, thick] (data) to [bend left=20] node[left]{} (hagent);
\draw[->, thick, dashed] (meta) -- node[above]{$l(.)$} (hagent);
\draw[->, thick, dashed] (meta) to [bend right=20] node[left]{$l(.)$} (agent);
\end{tikzpicture}
    \caption{\textbf{GMHF Schematic.} The AI-Agent modulates the latent space via $T_a(z)$ to synthesize training data $\mathcal{D}_{\text{gen}}$ for a Meta-Learner. A Human-Agent filters deployment outcomes and provides feedback $f_h$. The policy $\pi$ is optimised using a composite reward $R$ derived from human feedback and learner loss. Solid lines denote mandatory information flow; dashed lines denote optional transfers.}
    \label{fig:algo_flowchart}
\end{figure}
\section{Proposed Algorithm}

\begin{algorithm}[!ht]
\caption{GMHF Meta-Learner Optimization}
\label{alg: GMHF}
\DontPrintSemicolon
\SetAlgoLined

\SetKwInOut{Input}{Input}
\SetKwInOut{Output}{Output}

\Input{Data distribution $p(x|y)$, initial parameters}
\Output{Optimized $\pi^{*}$ and $\mathcal{M}_\psi^{*}$}
\BlankLine

Initialize cNODE model $\mathcal{V}$\;
Train cNODE on $p(x|y)$\;
Generate datasets $x_{\text{tr}}, x_{\text{dep}}$\;
Initialize AI-agent policy $\pi(a|z, f_h)$\;
\BlankLine

\While{not converged}{
    
    \For{each episode}{
        AI-agent selects action $a$ from $\pi(a|z, f_h)$\;
        Apply transformation $T_a$ on latent variables $z \to z'$\;
        Generate new data $x' \sim p_\theta(x|z', y)$\;
        Train meta-learner $\mathcal{M}_\psi$ on $x_{\text{tr}}$\;
        Compute accuracy and feedback $f_h$\;
        Compute reward $R(x') = \lambda_a \cdot \text{Acc} + \lambda_h \cdot f_h$\;
        Optimize AI-agent policy $\pi(a|z, f_h)$\;
    } %
} %
\Return{$\pi^{*}, \mathcal{M}_\psi^{*}$}\;
\end{algorithm}
Motivated by Theorem \ref{thm:final bound}, which bounds deployment risk via the divergence terms in Lemma \ref{eq:KL expression}, we introduce the GMHF algorithm. By optimising a policy against the human acceptance probability (Proposition \ref{prop:acceptance dist}), the agent iteratively steers the generative process toward the target distribution, explicitly minimising the generalisation gap identified in our theoretical analysis.

It provides a mechanism for the human, modelled as an agent, to alter the latent parameters such that the accepted distribution matches the target distribution. This implicitly accounts for human understanding of the target distribution. The outcome can vary depending on the algorithm's parameters and the human expert's knowledge.
The algorithm consists of (a) a generative model, (b) an AI-Agent, (c) A meta-learner, and (d) a human model with three different levels of expert knowledge (EK). A flowchart describing the algorithm is shown in Fig. \ref{fig:algo_flowchart} and the pseudocode is provided in Algorithm \ref{alg: GMHF}. 
\paragraph{Generative Model:} The proposed cNODE architecture serves as inductive bias for the neural ODE, the weights of which are initialised, and then trained on the source datasets $\{(x, ~y) ~s.~t.~ (x, ~y) \sim p(x \lvert y )p(y)\}$ to generate the distribution $\mathcal{D}_g$ with a density $p_g(.)$ conditioned on the parameters of the ODE (The architecture diagram and a brief description is provided in the Appendix. Fig. \ref{fig:cNODE_network}).
\paragraph{AI agent:}The AI agent interacting with the environment is modelled as an MDP defined by the tuple $\mathcal{M} = (\mathcal{S}, \mathcal{A},~ T_a,~ r,~ \rho_0,~ \gamma)$. The latent variable of the generative model, which in this case is the parameter $\lambda$, forms the states $\mathcal{S}$ and $\mathcal{A}$, representing the actions. The state transition is governed by the function $T_a: \mathcal{S \times\mathcal{A}} \longmapsto \mathcal{S}$. 
The human provides reward as feedback, and the reward function is denoted by the mapping $r : \mathcal{S} \times \mathcal{A} \times \mathcal{S} \mapsto \{-1, ~1\}$. The starting state distribution is $ \rho_0$, and $\gamma$ denotes the discount
factor. A policy $\pi: \mathcal{S} \longmapsto \mathcal{A}$ is defined as a mapping from the states to actions. In the case we consider, the latent variables of the generative model are transformed by $T_a$ for every action generated by the AI-agent.
\paragraph{Human Model.} 
We model human feedback as a stochastic reward $r \in \{-1, 1\}$ conditioned on the perceived system configuration $\tilde{\theta}$. This parameter vector $\tilde{\theta}$ represents the governing factors, known to the human. The reward follows a Bernoulli distribution biased by the expert's reliability $p_h$:
\begin{align}
p(r=1\mid \tilde{\theta})&=\begin{cases}
p_h, & \tilde{\theta} \in \Theta_{\text{valid}},\\[6pt]
1-p_h, & \text{otherwise},
\end{cases}\\
p(r=-1\mid\tilde{\theta})&=1-p(r=1\mid\tilde{\theta}).
\end{align}
This formulation incentivises the agent to discover the physically plausible configuration space $\Theta_{\text{valid}}$. Formally, the effective reward distribution marginalises over the human's perception noise:
$p(r \mid \theta) = \int p(r \mid \tilde{\theta})p(\tilde{\theta} \mid \theta) \, d\tilde{\theta}$. For experimental evaluation, we quantify the level of Expert Knowledge (EK) directly via the reliability parameter, setting $\text{EK} = p_h$.
\paragraph{Meta-learner:} We employ the Reptile algorithm \cite{nichol2018first} to train a regressor that predicts system parameters from trajectory data $(x_{tr}, y_{tr})$. The meta-learner's zero-shot prediction accuracy on the target domain serves as the primary metric for evaluating the utility of the generated datasets.
\subsection{Conditional Neural ODE (cNODE) as a Causal Proxy}
\label{ssec:cNODE}

To enable precise control over the generative process, we employ a Conditional Neural ODE (cNODE) framework. Unlike standard Neural ODEs that learn a fixed dynamics function, we explicitly parametrise the vector field as $f(\mathbf{x}, \tau, \bm{\lambda}, \bm{\beta})$, where $\bm{\beta}$ represents the shared physical laws (e.g., conservation of momentum) and $\bm{\lambda}$ captures instance-specific properties (e.g., stiffness). The state evolution is given by:
\begin{align}
\mathbf{x}(t) = \mathbf{x}(t_0) + \int_{t_0}^{t} f(\mathbf{x}, \tau, \bm{\lambda}, \bm{\beta}) \, d\tau
\end{align}
We train this model to minimize the trajectory loss $\mathcal{L}(\bm{\beta})$ across a distribution of simulation parameters $p(\bm{\lambda})$:
\begin{align}
\mathcal{L}(\bm{\beta}) := \mathbb{E} \left[ \ell \left( \mathbf{x}(t_0) + \int_{t_0}^{t} f(\mathbf{x}, \tau, \bm{\lambda}, \bm{\beta}) \, d\tau, \; \mathbf{x}_{sim}(t) \right) \right]
\end{align}
Here, $\mathbf{x}_{sim}(t)$ denotes the ground-truth trajectories obtained from the numerical simulator corresponding to parameters $\bm{\lambda}$, and $\ell(\cdot)$ is a distance metric.

While conditioning ODEs is a known technique in scientific computing, we leverage it here to create a \textit{disentangled action space} for the RL agent. By embedding $\bm{\lambda}$ as an input, we force the network to learn the differential equation's structural form while exposing $\bm{\lambda}$ as a ``knob''. This turns the black-box generative process into a causal proxy, allowing the agent to traverse the physical manifold in directions that are interpretable to the human expert.

\begin{figure}[!h]
    \centering
    \includegraphics[trim={0.3cm 0.3cm 0.3cm 0.3cm},clip,width=0.7\linewidth]{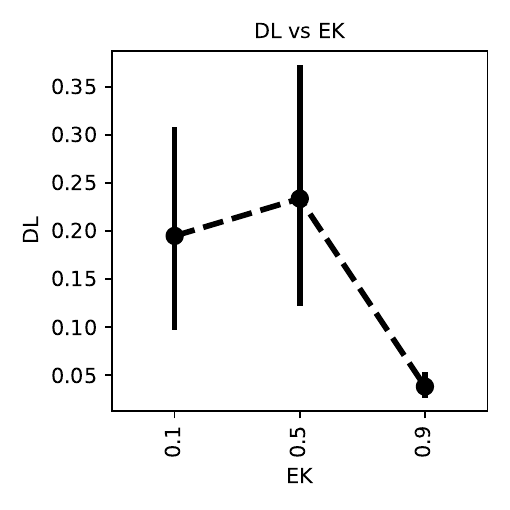}
    \caption{Deployment loss (DL) across expert-knowledge (EK) levels. EK values below 0.5 correspond to adversarial guidance, while values above 0.5 represent helpful human expertise. A marked decrease in deployment loss is observed as EK increases from 0.5 to 0.9, demonstrating the strong positive impact of supportive expert input on system performance.}
    \label{fig:EK}
\end{figure}

\begin{figure}[ht]
    \centering
    \begin{subfigure}{0.48\linewidth}
        \centering
        \includegraphics[width=\linewidth]{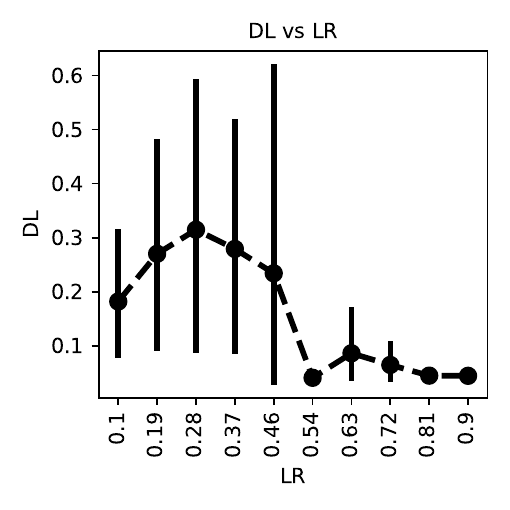}
        \caption{Meta-Learning Rate}
        \label{fig:LR}
    \end{subfigure}
    \hfill %
    \begin{subfigure}{0.48\linewidth}
        \centering
        \includegraphics[width=\linewidth]{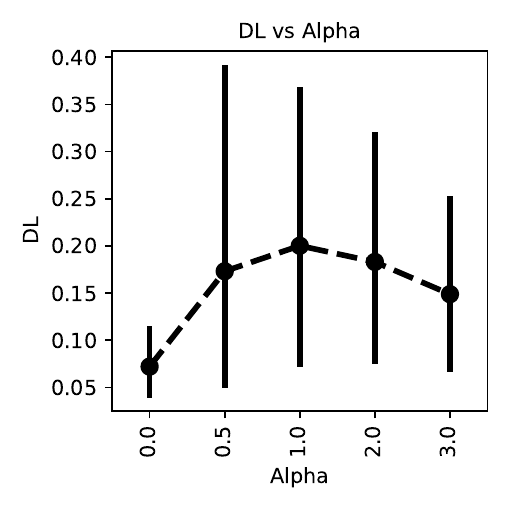}
        \caption{Nonlinearity ($\alpha$)}
        \label{fig:alpha}
    \end{subfigure}

    \caption{\textbf{Parameter Analysis} (a) Deployment loss (DL) minimizes at $LR \approx 0.54$, balancing adaptation speed with stability. (b) Loss varies non-monotonically with physical nonlinearity $\alpha$; it peaks at $\alpha \approx 1$ but decreases for stronger non-linearities ($\alpha > 1$), where high stiffness constrains the effective state-space.}
    \label{fig:sensitivity_analysis}
\end{figure}
\section{Experimental Results}
\label{sec:exp_results}
\paragraph{Model Problem} We evaluate our method on the Duffing oscillator, a standard benchmark for non-linear dynamics. The system evolution is governed by:
\begin{align}
    \dot{\mathbf{x}} = \underbrace{\begin{bmatrix} 0 & 1 \\ -k/m & -c/m \end{bmatrix}}_{\mathbf{K}} \mathbf{x} + \alpha \underbrace{\begin{bmatrix} 0 \\ x^3 \end{bmatrix}}_{\mathbf{f}(\mathbf{x})}
    \label{eq:spring_damper_eom}
\end{align}
where $\mathbf{x} = [x, \dot{x}]^\top$ represents the state (position and velocity). The mass is fixed ($m=1$), while the stiffness $k$ and damping $c$ constitute the parameter set $\bm{\lambda} = \{k, c\}$ used to generate diverse datasets. The coefficient $\alpha$ controls the system's nonlinearity; we vary $\alpha \in [0, 3]$ in increments of 0.5 to evaluate the framework's performance across regimes from linear to highly nonlinear.

We parameterise the vector field $f(\mathbf{x}, t, \bm{\lambda}, \bm{\beta})$ using a four-layer Multi-Layer Perceptron (MLP) with 64 hidden units per layer. To construct the training corpus, we sampled physical parameters from uniform distributions—specifically stiffness $k \sim \mathcal{U}[0, 10]$ and damping $c \sim \mathcal{U}[0, 1]$—and numerically integrated the dynamics (refer to the Eq. \ref{eq:spring_damper_eom}) with a fixed mass $m=1$.

Our design emulates a realistic deployment setting in which analytical governing equations are inaccessible, and the system must be learned solely from observational data. In this context, the cNODE acts as a data-driven surrogate model (or "Digital Twin"), capable of synthesising physically consistent trajectories under expert supervision. The model successfully captures diverse oscillator behaviours governed by the stiffness and damping parameters (See Appendix  Fig. \ref{fig:ts_node} for the samples). Duffing oscillator serves as a canonical proxy for complex real-world applications, modelled by damped harmonic oscillators, ranging from physiological applications to (e.g., cardiac rhythms) to structural health monitoring in mechanical systems \cite{fonkou2023heart, yin2006structural}.

We trained the agent using the Deep Deterministic Policy Gradient (DDPG) algorithm~\cite{lillicrap2015continuous}, implemented via Stable Baselines3~\cite{stable-baselines3}. The task is formulated as a continuous control problem with one-dimensional state and action spaces ($\mathcal{S}, \mathcal{A} \subseteq \mathbb{R}$). The state $s_t$ corresponds to the current stiffness parameter $k$, (refer to the Eq. \ref{eq:spring_damper_eom}), and the agent learns a deterministic policy $\pi(s_t)$ to adjust $k$ toward the target configuration. To isolate the impact of expert guidance, we set the weight $\lambda_a=0$ and feedback weight $\lambda_h=1$, effectively optimising the policy solely against the stochastic binary reward $r \in \{-1, 1\}$ from the human model.
\paragraph{Main Result: Impact of Expert Knowledge} 
The human model parameters $\tilde{\theta}$  is set to $k$, and the set $\Theta_{valid} = [0, ~10]$. To isolate the contribution of human guidance, we evaluate deployment loss across varying levels of expert knowledge (EK). As shown in Figure~\ref{fig:EK}, the system is highly sensitive to feedback quality. We observe that under adversarial ($EK=0.1$) or random ($EK=0.5$) guidance, the loss remains high with significant variance. However, improving reliability to $EK=0.9$ triggers a sharp transition, reducing loss considerably (to $<0.06$). This confirms that a threshold of high-fidelity expert feedback is necessary to successfully steer the generative process toward the target distribution. The outliers are removed before creating the plots in Fig. \ref{fig:EK}, \ref{fig:LR} and Fig. \ref{fig:alpha} for clarity.

\paragraph{Meta-Update} 
We analyse the sensitivity of the meta-learning rate $LR$, which governs the Reptile-based update rule: $\bm{\beta} \leftarrow \bm{\beta}+ LR(\tilde{\bm{\beta}} - \bm{\beta})$. As shown in Figure~\ref{fig:LR}, the system exhibits a sharp transition rather than a standard convex trade-off. At lower rates ($LR < 0.5$), the model suffers from high deployment loss and significant variance, indicating a failure to retain the physical constraints identified during the inner loop. We observe a critical threshold at $\beta \approx 0.54$, where the loss collapses to its minimum ($<0.05$). Beyond this point, the performance remains robustly stable, confirming that aggressive meta-updates are required to effectively bridge the domain gap.
\paragraph{Effect of System Nonlinearity} 
We further analysed the robustness of our method against the physical complexity of the underlying dynamics, controlled by the cubic nonlinearity parameter $\alpha$ (see Eq. \ref{eq:spring_damper_eom}). As shown in Figure~\ref{fig:alpha}, the linear case ($\alpha = 0$) presents the lowest difficulty, resulting in minimal deployment loss. As the system introduces nonlinearity, the prediction task becomes increasingly challenging, peaking at $\alpha \approx 1$ where the deviation from simple harmonic motion is most pronounced. Interestingly, for stronger non-linearities ($\alpha > 1$), the loss begins to decrease. This suggests that while moderate nonlinearity introduces complex anharmonic dynamics, significantly higher stiffness terms effectively constrain the system's effective state space by rapidly restoring the oscillator to equilibrium, acting as a stabilising force that facilitates easier adaptation for the meta-learner.

We further show, on the Duffing oscillator, that the agent reduces the generated-to-target divergence under reliable feedback (Appendix~\ref{sec:divergence}). Additional experiments on a non-dynamical probabilistic model demonstrate that GMHF is not restricted to ODE-governed time series (Appendix~\ref{sec:pgm_benchmark}).

\section{Discussion and Conclusion}
\label{sec:discussion}
The experimental results validate GMHF as a robust framework for bridging domain gaps in scientific machine learning. Inspired by the theoretical bounds from Section~\ref{sec:math_modelling} and with the empirical evidence obtained in Section \ref{sec:exp_results}, we derive critical insights into how the agent learns to collaborate with human experts.

The core innovation of GMHF is its fundamental shift in objective: rather than passively fitting a static dataset, the algorithm learns how to collaborate to synthesise a superior one. Unlike traditional surrogates that minimise reconstruction error against a fixed ground truth, our agent learns a policy $\pi(a|z, f_h)$ that actively manipulates latent parameters to satisfy a human utility function. This paradigm is theoretically anchored in Proposition~\ref{eq:correct_gen_dist_gaus}, which demonstrates that accurate target approximation requires the generator to align its density with the human's acceptance constraints. In our specific experimental setting ($\lambda_a=0$), the agent does not correct for human bias but rather converges to the expert's implicit physical model. The sharp reduction in loss at high reliability levels ($EK > 0.5$, Fig.~\ref{fig:EK}) empirically validates this mechanism: the agent leverages sparse feedback to navigate the latent space, effectively solving an analytically intractable inverse search problem for the expert. This establishes a collaborative equilibrium where the machine provides exhaustive exploration and the human provides physical verification—a synergy that achieves results neither could attain in isolation.

Our analysis of system nonlinearity points to a link between modelling parameters and statistical variance. The reduction in loss under strong nonlinearity ($\alpha > 1$, Fig.~\ref{fig:alpha}) can be explained via the isotropic KL divergence (Eq.~\ref{eq:kl_gaus}). High stiffness implies that the target distribution variance is reduced, and this also reduces the distribution of the acceptance distribution, which humans are implicitly aware of. This, among other effects, can amplify the penalty for the mean mismatch in Eq.~\ref{eq:kl_gaus}, increasing the loss. These low-variance regimes (very high and very low nonlinearity) can provide a steeper gradient for the meta-learner to learn efficiently, compared to the high-variance regime of the nonlinear system (medium nonlinearity).

The ablation study on meta-learning rates reveals a critical threshold behaviour rather than a trade-off. As shown in Figure~\ref{fig:LR}, conservative update rates ($LR < 0.5$) result in high deployment loss and significant variance. In our Reptile-based implementation, the meta-update acts as a gating mechanism for new physical knowledge; if $LR$ is too low, the agent effectively discards the valid constraints identified during the inner loop, failing to bridge the domain gap. We observe a transition at $LR \approx 0.54$, where the loss collapses to a minimum. This indicates that the meta-learner requires aggressive updates to effectively retain the expert's directional signal.

Two further results reinforce these conclusions. First, tracking the divergence between generated and deployment data on the Duffing oscillator (Appendix~\ref{sec:divergence}, Fig.~\ref{fig:wass_dist}) shows that reliable feedback drives this divergence down over training. This provides direct empirical support for the divergence-minimisation mechanism that Theorem~\ref{thm:final bound} bounds in principle, indicating that the theory informs the algorithm. Second, applying GMHF without architectural modification to a non-dynamical probabilistic model with an explicit source-to-target extrapolation gap (Appendix~\ref{sec:pgm_benchmark}) confirms that the approach is not specific to ODE-governed dynamics. Moreover, the meta-learning rate exhibits opposite optimal regimes across the two benchmarks, indicating that the framework adapts to the differing demands of each task rather than imposing a single fixed behaviour (Fig.~\ref{fig:pgm_DL_mlr}).
\section{Limitations and Future Work}
A primary limitation of our study is that, while it spans both a dynamical benchmark (the Duffing oscillator) and a non-dynamical one (the linear-cubic probabilistic model), both are simulated systems with low-dimensional latent parameters and a known ground truth. We adopted such systems because they permit rigorous validation of the algorithm against a controllable target. The framework is adaptable in principle to richer modalities. Future work could extend the approach to image or tabular data by replacing the generator with a domain-appropriate causal structure, such as a diffusion model for images or a Structural Causal Model (SCM) for tabular reasoning. A key challenge in those settings is that a controllable generator with a meaningful latent parameter, together with a reliable ground truth for training it, may be harder to obtain than in the synthetic systems studied here.

While GMHF successfully bridges the generalisation gap, it relies on two key assumptions: (1) the existence of a shared physical model between source and target domains (e.g., both are oscillators), and (2) the availability of a non-adversarial expert ($EK > 0.5$). Our experiments show that performance degrades sharply if the expert's reliability falls below random chance. Furthermore, the current interaction protocol assumes a stationary target; extending this to time-varying systems where parameters drift during deployment remains an open challenge for future research.

This work opens the door for more nuanced human-machine collaboration. While our current model relies on binary feedback and a probabilistic utility function, future iterations will incorporate complex human models capable of providing multimodal feedback (e.g., natural-language critiques of physical plausibility). Ideally $\tilde{\theta}$ depends on the observed trajectory, we simplify $\tilde{\theta} \approx k$ for experiments  (Eq.~\ref{eq:spring_damper_eom}). In future work, one could build a more realistic human model that generates the parameters $\tilde{\theta}$ from observations of trajectories. Furthermore, extending the ``stiffness as constraint'' insight to high-dimensional systems, such as fluid dynamics or cardiac modelling, remains a promising avenue for making intractable inverse problems solvable via human-guided meta-learning. 

\section*{Acknowledgments and Disclosure of Funding}
This work was supported by UKRI Turing AI World-Leading Researcher Fellowship (EP/W002973/1). SK was also supported by EU funding  ERC ODD-ML 101201120 and the Research Council of Finland Flagship programme: Finnish Center for Artificial Intelligence FCAI.

\bibliographystyle{plainnat}
\bibliography{refs}

\appendix
\section{Appendix}
\subsection{Proofs}
\label{Appendix:Proofs}

\begin{lemma}
\begin{align}
    D_{KL}(p_t \| p_a) &= 
\underbrace{\delta_h}_{\text{Human belief mismatch}} \nonumber \\ 
&-\underbrace{L_h}_{\text{Human latent inference}} \nonumber \\
&-\underbrace{L_g}_{\text{Generator data fit}} \nonumber \\ 
&+\underbrace{\log z}_{\text{Normalization}}
\end{align}

where,
\begin{align}
   L_h &:=\mathbb{E}_{p_t}[\log p_h(\tilde{\theta}|x)] \\
   L_g &:=\mathbb{E}_{p_t}[\log p_g(x|\theta)]\\
   \delta_h &:=D_{KL}(p_t \| \tilde{p}_t)
\end{align}

\end{lemma}
\begin{proof}
    \begin{align}
        D_{KL}(p_t (x) \mid \mid p_a(x)) 
        &=
        \int p_t(x) \log \frac{p_t (x)}{p_a(x)}  dx
    \end{align}
    This could be further expanded as,
    \begin{align}
        &= \int p_t(x) \log \frac{p_t(x) z}{p_h(\tilde{\theta} \lvert x)\tilde{p}_t(x)p_g(x \lvert \theta)} dx\\
        &= \int p_t(x) ( \log p_t(x) - \log p_h(\tilde{\theta} \lvert x) \nonumber \\
        &\quad - \log \tilde{p}_t(x) - \log p_g(x \lvert \theta) + \log z ) dx\\
        &= D_{KL}(p_t(x) \| \tilde{p}_t(x)) - \mathbb{E}_{p_t} \log p_h(\tilde{\theta} \lvert x) \nonumber \\
        &\quad - \mathbb{E}_{p_t} \log p_g(x \lvert \theta) + \log z
    \end{align}
\end{proof}

\begin{lemma}[Bounding the expectation term]
Let $p_t(x)$ be the target distribution and $p_a(x)$ the acceptance distribution as defined in previous assumptions. Let $h: \mathcal{X} \to \mathbb{R}$ be a bounded function such that $|h(x) - f_t(x)| \leq M$ for all $x \in \mathcal{X}$. Then, the expectation mismatch term:
\begin{align}
    \Delta := \mathbb{E}_{p_t}[(h(x) - f_t(x))] - \mathbb{E}_{p_a}[(h(x) - f_t(x))]
\end{align}
can be bounded as follows:
\begin{align}
    |\Delta| \leq M \sqrt{2 D_{KL}(p_t \| p_a)}.
\end{align}
\end{lemma}
\begin{proof}
\vspace{-10pt}
We start by expressing the expectation mismatch in terms of the total variation distance:
\begin{align}
    |\mathbb{E}_{p_t}[h(x) - f(x)] - \mathbb{E}_{p_a}[h(x) - f(x)]| \leq M \cdot 2 \cdot \text{TV}(p_t, p_a),
\end{align}
where the total variation distance is given by:
\begin{align}
    \text{TV}(p_t, p_a) = \frac{1}{2} \int |p_t(x) - p_a(x)| dx.
\end{align}
Applying Pinsker’s inequality:
\begin{align}
    \text{TV}(p_t, p_a) \leq \sqrt{\frac{1}{2} D_{KL}(p_t \| p_a)},
\end{align}
we obtain the final bound:
\begin{align}
    |\Delta| \leq M \sqrt{2 D_{KL}(p_t \| p_a)}.
\end{align}
\end{proof}

\begin{theorem}
For a hypothesis $h$, the target error $\epsilon_t(h)$ and the source error $\epsilon_s(h)$, one could modify the bound for the target error as follows, where, $h: \mathcal{X} \to \mathbb{R}$ be a bounded function such that $|h(x) - f_t(x)| \leq M$.
    \begin{align}
        \epsilon_t(h) &\leq  \epsilon_s(h) +\min \{ \mathbb{E}_{p_a}[\mid f_s(x) - f_t(x)\mid ], \nonumber \\
        &\quad \mathbb{E}_{p_t}[\mid f_s(x) - f_t(x)\mid ]\} \nonumber \\
        &\quad+M\sqrt{2 (\delta_h - L_h - L_g + \log z)}
    \end{align}

    \begin{proof}
    From lemma \ref{eq:KL expression}, \ref{eq: shai dist shift} and \ref{eq: bound expectation difference}
    \begin{align}
        \epsilon_t(h) &\leq \epsilon_s(h) + \min \{ \mathbb{E}_{p_a}[\mid f_s(x) - f_t(x)\mid ],\nonumber\\&\quad \mathbb{E}_{p_t}[\mid f_s(x) - f_t(x)\mid ]\} \nonumber\\&\quad +\mid \mathbb{E}_{p_t}(h(x)-f_t(x)) - \mathbb{E}_{p_a}(h(x)-f_t(x)) \mid\\ \nonumber
        &\leq \epsilon_s(h) + \min \{ \mathbb{E}_{p_a}[\mid f_s(x) - f_t(x)\mid ],\nonumber\\&\quad \mathbb{E}_{p_t}[\mid f_s(x) - f_t(x)\mid ]\}\\ &\quad +M\sqrt{2 D_{KL}(p_t \| p_a)} \label{eq:useful_for_gaus_bound}\\
        &= \epsilon_s(h) +\min \{ \mathbb{E}_{p_a}[\mid f_s(x) - f_t(x)\mid ],\nonumber\\&\quad \mathbb{E}_{p_t}[\mid f_s(x) - f_t(x)\mid ]\} \nonumber \\ &\quad+M\sqrt{2 (\delta_h - L_h - L_g + \log z)}
    \end{align}
    \end{proof}
\end{theorem}

\subsection{Architecture of the cNODE and Samples}
\label{sec:cnode_arch}
This section provides the architecture diagram of the cNODE (Appendix Fig. \ref{fig:cNODE_network}) and the trajectories it generates (Appendix Fig. \ref{fig:ts_node}).
\begin{figure}[!h]
    \centering
    \includegraphics[trim={0.8cm 1.5cm 0.8cm 1.6cm},clip, height=8cm]{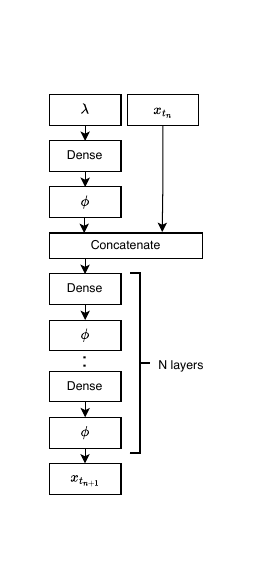}
    \caption{Architecture of the Conditional Neural ODE (cNODE). The network learns the vector field conditioned on system parameters $\bm{\lambda}$, allowing the generation of trajectories that respect varying physical properties}
    \label{fig:cNODE_network}
\end{figure}
\label{ssec:Simulator}
\begin{figure}[h]
    \centering
    \includegraphics[width=3.5in]{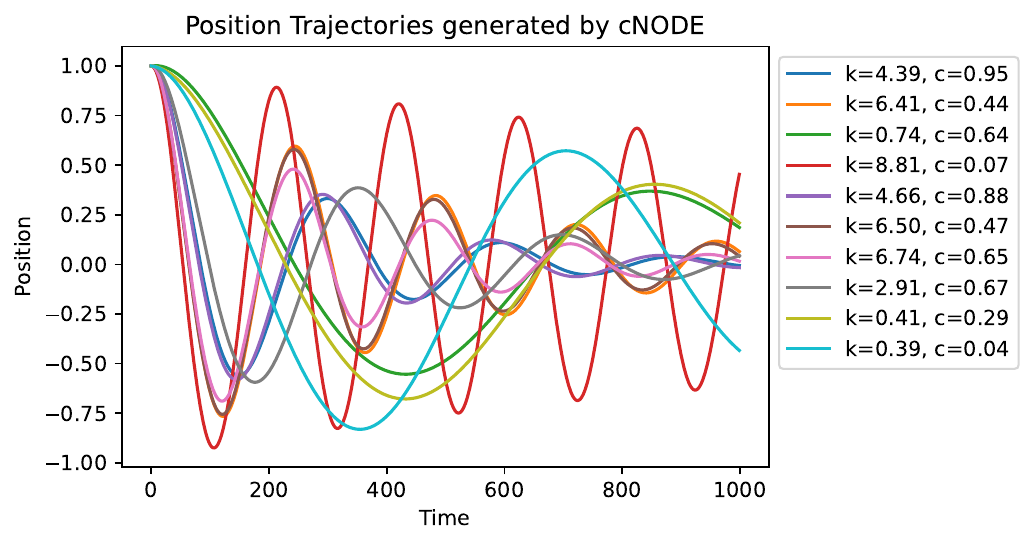}
    \caption{Position trajectories generated by the cNODE for varying stiffness ($k$) and damping ($c$) parameters. The distinct curves demonstrate the model's ability to generalise across different physical regimes while keeping mass fixed at $m=1$ and $\alpha=0.5$.}
    \label{fig:ts_node}
\end{figure}
\section{Divergence Analysis}
\label{sec:divergence}

The theoretical analysis (Lemma~\ref{eq:KL expression}) bounds deployment error in terms of the divergence $D_{KL}(p_t \| p_a)$ between the target distribution and the distribution of human-accepted samples. The GMHF algorithm is constructed to minimise this divergence implicitly. Under reliable feedback, the agent is expected to guide the generative process so that accepted samples increasingly resemble the deployment distribution. To assess whether this mechanism operates in practice, and to determine whether the theory genuinely informs the algorithm rather than merely motivating it, the divergence between the generated data and the deployment distribution is tracked throughout training.

Figure~\ref{fig:wass_dist} presents the Wasserstein distance between generated samples at each episode and a reference sample from the deployment distribution, under both reliable ($EK = 0.9$) and uninformative ($EK = 0.5$) expert feedback. With reliable feedback, the divergence decreases over training as the agent learns to generate data aligned with the deployment regime. In contrast, under uninformative feedback, the divergence remains high and shows no consistent downward trend. These results provide direct empirical evidence that the human feedback channel guides the generative process toward the target distribution, consistent with the divergence-minimisation mechanism identified in the theoretical analysis.

\begin{figure}
    \centering
    \includegraphics[width=0.7\linewidth]{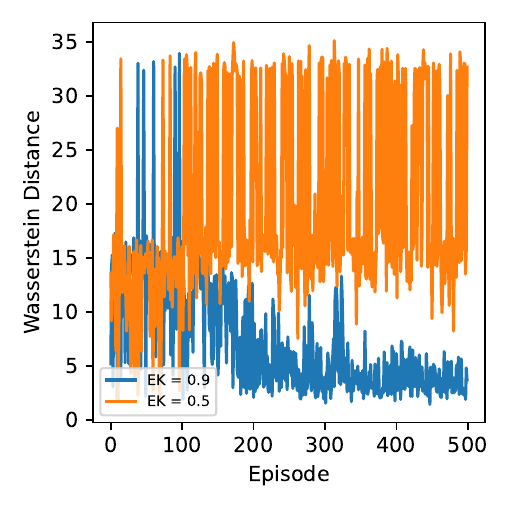}
    \caption{Wasserstein distance between generated and deployment data as a function of training episode, for reliable ($EK = 0.9$) and uninformative ($EK = 0.5$) expert feedback. Reliable feedback drives the generated distribution toward the deployment distribution, reducing the divergence; uninformative feedback does not.}
    \label{fig:wass_dist}
\end{figure}
\section{Additional experiments: A probabilistic model}
\subsection{Linear-cubic probabilistic model benchmark}
\label{sec:pgm_benchmark}

To complement the dynamical Duffing oscillator and demonstrate that GMHF is not tied to ODE-governed systems, we introduce a second benchmark based on a static probabilistic model with a coupled latent structure. Each task is parametrised by a task-level latent slope $m$ and offset $c$, and consists of $N = 20$ paired observations $(y_i, z_i)_{i=1}^{N}$ generated from a shared latent input $x_i$. The full generative model is
\begin{align}
m &\sim \mathcal{U}(m_a, m_b) \\
c &\sim \mathcal{U}(c_a, c_b) \\
x_i &\sim \mathcal{N}(0, \sigma_x^2) \\
\varepsilon_i^{(y)}, \, \varepsilon_i^{(z)} &\sim \mathcal{N}(0, \sigma_\varepsilon^2) \\
y_i &= m \, x_i + c + \varepsilon_i^{(y)}
 \label{eq:y_i}\\
z_i &= m^3 \, x_i + c + \varepsilon_i^{(z)}
\label{eq:z_i}
\end{align}
for $i = 1, \ldots, N$, with $\sigma_x = 2$ and $\sigma_\varepsilon = 0.5$. The meta-learner's task is to recover the latent input $x_i$ from the observation pair $(y_i, z_i)$, treating $m$ as the latent regime parameter modulated by the RL agent under human feedback, and $c$ as a nuisance variable that must be implicitly marginalised.

\paragraph{Source and target distributions.}
We instantiate the source and target regimes so that deployment lies strictly outside the source support of $m$:
\begin{align}
\text{Source:} \quad & m \sim \mathcal{U}(0.5, 1.0), \quad c \sim \mathcal{U}(0, 20) \\
\text{Target:} \quad & m \sim \mathcal{U}(1.5, 1.7), \quad c \sim \mathcal{U}(0, 20)
\end{align}
This creates a genuine extrapolation gap: a meta-learner exposed only to source-regime slopes must rely on the GMHF agent, steered by human feedback, to generate training data at slope values the source distribution would not produce. The offset $c$ is shared across regimes and acts purely as a nuisance variable.

\paragraph{Identifiability under regime shift.}
Recovering $x_i$ requires eliminating the unknown offset $c$, achieved by subtracting Eq.~\ref{eq:y_i} from Eq.~\ref{eq:z_i}:
\begin{align}
z_i - y_i &= (m^3 - m) \, x_i + \big( \varepsilon_i^{(z)} - \varepsilon_i^{(y)} \big),
\end{align}
so the signal carrying $x_i$ scales with $|m^3 - m|$. This coefficient is small in the source regime and substantially larger at the target, so the optimal $(y_i, z_i) \to x_i$ mapping differs between regimes. A meta-learner trained only on source data must therefore extrapolate, unless the agent is guided to explore the target regime --- the role of the human feedback channel.

\paragraph{The generative model}
We note that for this benchmark the generative model shares the exact functional form of the true data-generating process; the outcome is therefore independent of a misspecified generator and the agent instead steers the latent slope $m$ toward the deployment regime. This isolates the latent-steering role of GMHF from the additional challenge of generator misspecification present in the Duffing setting, where the cNODE is a learned approximation of the true dynamics.
\paragraph{Reference conditions via expert-reliability ablation.}
Rather than introducing external baselines that would require different architectural assumptions, we use the expert reliability $EK$ itself to define reference conditions, keeping every component of the system identical except the reward signal. At $EK = 0.5$ the human feedback carries no information about the underlying parameter, so the system reduces to GMHF with effectively random reward, a natural no-information baseline. At $EK = 0.1$ the feedback is consistently inverted, and at $EK = 0.9$ the human supplies reliable guidance. Figure~\ref{fig:pgm_DL_EK} reports deployment loss across these three regimes: high expert reliability ($EK = 0.9$) yields substantially lower deployment loss than the no-information condition ($EK = 0.5$), isolating the contribution of expert knowledge.

Figure~\ref{fig:pgm_DL_mlr} shows the sensitivity of deployment loss to the meta-learning rate. Here the optimal regime is reversed relative to the Duffing oscillator: low rates yield the lowest loss, and performance degrades at higher rates. Both settings exhibit a sharp transition, but in opposite directions. This reversal is consistent with the differing demands of the two tasks: the Duffing task rewards retaining per-task structure (favouring aggressive updates), whereas the probabilistic task requires averaging across tasks to marginalise the shared nuisance offset $c$ (favouring conservative updates).

\begin{figure}[t]
    \centering
    \includegraphics[width=0.7\linewidth]{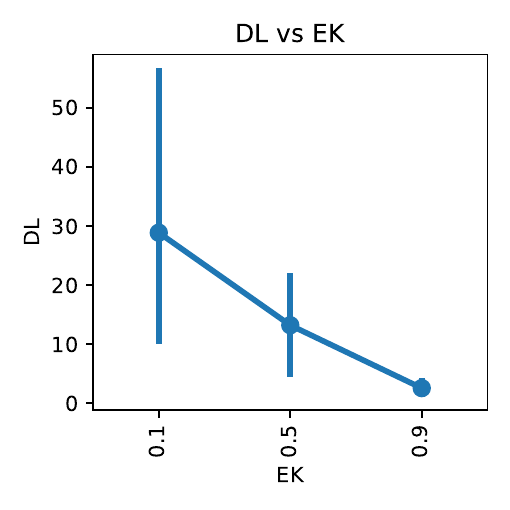}
    \caption{Deployment loss versus expert knowledge ( $EK= p_h$) on the linear-cubic probabilistic benchmark. Reliable expert feedback ($EK = 0.9$) achieves the lowest deployment loss, while random feedback ($EK = 0.5$) provides no useful signal. Points show the mean across meta-learning-rate settings; bars indicate variability.}
    \label{fig:pgm_DL_EK}
\end{figure}

\begin{figure}[!h]
    \centering
    \includegraphics[width=0.7\linewidth]{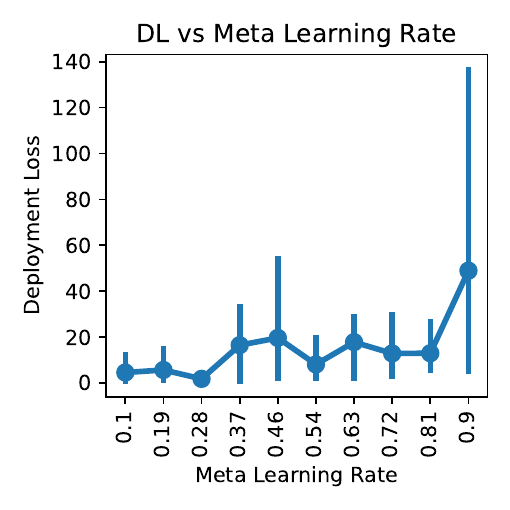}
    \caption{Deployment loss versus meta-learning rate on the linear-cubic probabilistic benchmark. In contrast to the Duffing oscillator (Fig.~\ref{fig:LR}), performance is best at low meta-learning rates and degrades at higher rates. Both benchmarks exhibit a sharp transition, but in opposite directions. Points show the mean across expert-knowledge settings; bars indicate variability.}
    \label{fig:pgm_DL_mlr}

\end{figure}

\paragraph{Additional experimental details}
The agent is implemented using the Deep Deterministic Policy Gradient (DDPG) algorithm~\citep{lillicrap2015continuous} via the Stable-Baselines3 library~\citep{stable-baselines3}, with an MLP policy with 2 layers of 64 neurons each. The task is formulated as a continuous control problem with one-dimensional state and action spaces. The state $s_t$ corresponds to the current latent parameter (stiffness $k$ for the Duffing oscillator; slope $m$ for the linear-cubic model), and the deterministic transition is $s_{t+1} = s_t + a_t$, where $a_t$ is the agent's action. The action space is bounded to $[-1, 10]$ for the Duffing oscillator and $[-0.5, 0.5]$ for the linear-cubic model; the observation space is $[-100, 100]$. Each episode terminates when the state exits the valid range or the human returns a positive reward. The agent is trained for $500$ environment steps. To isolate the contribution of human feedback, we set the accuracy weight $\lambda_a = 0$ and the feedback weight $\lambda_h = 1$, so the reward reduces to the binary human signal $r \in \{-1, +1\}$. 

For the linear-cubic model, the human returns a positive reward when the slope lies in the target band $1.4 < m < 1.8$, and each episode terminates when $m$ exits the range $[0.1, 2.5]$.

\end{document}